%% file: matrix_mirror_descent.tex
\newcommand{\CLASSINPUTtoptextmargin}{0.8in}
\documentclass[conference]{IEEEtran}
\IEEEoverridecommandlockouts
\usepackage{amsthm}
\usepackage{cite}
\usepackage{amsmath,amssymb,amsfonts,bm}
\usepackage{algorithmic}
\usepackage{graphicx}
\usepackage{textcomp}
\usepackage{xcolor}
\def\BibTeX{{\rm B\kern-.05em{\sc i\kern-.025em b}\kern-.08em
    T\kern-.1667em\lower.7ex\hbox{E}\kern-.125emX}}
    
\begin{document}

\title{Implicit Bias and Convergence of Matrix Stochastic Mirror Descent\\}

\author{\IEEEauthorblockN{Danil Akhtiamov$^*$}
\IEEEauthorblockA{\textit{Computing+Mathematical Sciences}\\
\textit{Caltech} \\
Pasadena, CA\\
dakhtiam@caltech.edu}
\and
\IEEEauthorblockN{Reza Ghane$^*$}
\IEEEauthorblockA{\textit{Electrical Engineering}\\
\textit{Caltech} \\
Pasadena, CA\\
rghanekh@caltech.edu}
\and
\IEEEauthorblockA{ Omead Pooladzandi\\
\textit{Electrical Engineering}\\
\textit{Caltech} \\
Pasadena, CA\\
omead@caltech.edu}
\and
\IEEEauthorblockN{Babak Hassibi}
\IEEEauthorblockA{\textit{Electrical Engineering}\\
\textit{Caltech} \\
Pasadena, CA\\
hassibi@caltech.edu}
}

\maketitle

\def\thefootnote{*}\footnotetext{Equal contribution}

\input{symbols}
\begin{abstract}

 We investigate Stochastic Mirror Descent (SMD) with matrix parameters and vector-valued predictions, a framework relevant to multi-class classification and matrix completion problems. Focusing on the overparameterized regime, where the total number of parameters exceeds the number of training samples, we prove that SMD with matrix mirror functions $\psi(\cdot)$ converges exponentially to a global interpolator. Furthermore, we generalize classical implicit bias results of vector SMD by demonstrating that the matrix SMD algorithm converges to the unique solution minimizing the Bregman divergence induced by $\psi(\cdot)$ from initialization subject to interpolating the data. These findings reveal how matrix mirror maps dictate inductive bias in high-dimensional, multi-output problems.

\end{abstract}

\begin{IEEEkeywords}
Stochastic Mirror Descent, Convergence, Implicit Bias, Nuclear Norm, Matrix Completion, Data Imputation
\end{IEEEkeywords}

\section{Introduction}

The choice of optimization algorithm plays a crucial role in determining not only convergence speed but also the properties of learned models in overparameterized machine learning. While gradient descent and its variants have become the main workhorse in large-scale optimization, recent theoretical insights reveal that the geometry induced by different optimizers leads to fundamentally different solutions. This phenomenon, known as implicit bias, has sparked renewed interest in understanding how algorithmic choices shape the learning process beyond mere convergence guarantees.

Stochastic mirror descent (SMD) generalizes standard gradient descent by performing updates in a dual space induced by a mirror map $\nabla \psi$, where $\psi: \mathbb{R}^p \to \mathbb{R}$ is a strongly convex potential function. Specifically, for a given training loss  $\calL_t(\mathbf{w}_t)$ computed with respect to the batch sampled at the time step $t$, step size $\eta$, and parameters $\mathbf{w}_t \in \mathbb{R}^p$ at iteration $t$, SMD performs the following update:
\begin{align}\label{eq: smd}
    \nabla \psi(\mathbf{w}_{t+1}) = \nabla \psi(\mathbf{w}_t) - \eta \nabla_{\mathbf{w}} \calL_t(\mathbf{w}_t) 
\end{align}

Setting $\psi(\cdot) = \frac{1}{2}\|\cdot\|_2^2$ recovers standard SGD. The power of this framework lies in its flexibility: different potential functions $\psi$ encode different geometries into the optimization dynamics. For overparameterized problems where multiple solutions interpolate the training data perfectly, SMD exhibits an implicit bias property - it converges to the solution minimizing the Bregman divergence $D_\psi(\mathbf{w}, \mathbf{w}_0)$ to the initialization $\mathbf{w}_0$ among all global minimizers, where
\begin{equation*}
D_\psi(\bw, \bw_0) = \psi(\bw) - \psi(\bw_0) -  \nabla \psi(\bw_0)^T( \bw-\bw_0)
\end{equation*}

In particular, when initialized near zero $\mathbf{w}_0 \approx 0$, SMD converges to the interpolator minimizing $\psi(\mathbf{w})$ among all interpolating solutions. This \emph{implicit bias property} of stochastic mirror descent implies that SGD finds the minimal $\ell_2$-norm solution among interpolators.

Various works have established convergence and implicit bias for linear models with vector parameters and scalar labels, capturing the linear regression and linear classification tasks \cite{gunasekar2018characterizing, azizan2018stochastic, Varma2025Exponential, sun2022mirror, azizan2021stochastic, soudry2018implicit, ji2019implicit}. In particular, \cite{gunasekar2018characterizing} characterized the implicit bias of vector mirror descent, while \cite{azizan2018stochastic} proved the convergence of vector Stochastic Mirror Descent (SMD). More recently, \cite{Varma2025Exponential} established an exponential convergence rate for this setting. Furthermore, \cite{azizan2021stochastic} extended the analysis of \cite{soudry2018implicit} and \cite{ji2019implicit} regarding classification margins from gradient descent to the mirror descent framework. However, these works treat parameters as unstructured vectors, potentially missing geometric properties encoded in their matrix representation. We therefore extend this framework to matrix weights and vector outputs, motivated by the observation that a plethora of problems in modern signal processing and data science, such as the matrix completion problem, are naturally formulated as problems of finding matrices satisfying certain structural properties. Namely, we consider the following update rule for $\bW_t \in \bbR^{d \times k}$, which we refer to as Matrix SMD:
\begin{align*}
    \nabla \psi(\bW_t) = \nabla \psi(\bW_{t-1}) - \eta \nabla_{\bW} \calL_t(\bW_{t-1})
\end{align*}

After establishing convergence and implicit bias guarantees, we demonstrate a practical application of the matrix SMD. Using the mirror function $\psi(\mathbf{W}) = \sum_{i=1}^m \sigma_i(\mathbf{W})^{p}$, where $\sigma_i(\mathbf{W})$ denotes the $i$-th singular value of $\mathbf{W}$, we fit a linear model to solve the matrix completion problem.  Setting $p \approx 1$ to approximate the nuclear norm yields a low-rank solution, which is a standard hypothesis for the matrix completion task. We demonstrate empirically that Matrix SMD leads to a lower error than standard singular value thresholding methods, which are usually used for minimizing the nuclear norm in practice.

\section{Notation and Problem Formulation}

\subsection{The problem}

We consider the problem of recovering a matrix $\bW \in \bbR^{d \times k}$ subject to linear constraints. 

\begin{definition}[Linear Constraint System]
Let $\calA: \bbR^{d \times k} \to \bbR^p$ be a linear operator with matrix representation $\bA = (\ba_1,\dots,\ba_p)^T$ where each $\ba_i \in \bbR^{d \times k}$ is treated as a vector. The constraint system is:
$$\calA(\bW) = \bb \quad \text{equivalently} \quad \bA \textbf{vec}(\bW) = \bb$$
where $\bb \in \bbR^p$ is a known vector. In the current exposition, we assume that $d\times k > p$, a regime commonly referred to as the overparameterized regime.
\end{definition}

\begin{example}[Matrix Completion]\label{ex: mtrx_compl}
In matrix recovery, we observe a subset $\Omega = \{(i_1,j_1),\dots,(i_p,j_p)\} \subset [d] \times [k]$ of matrix entries. Here:
\begin{itemize}
    \item $p = |\Omega|$ is the number of observed entries
    \item $\calA(\bW)_q = \bW_{i_q, j_q}$ extracts the $(i_q, j_q)$-th entry
    \item $\bb$ contains the observed values at positions $\Omega$
\end{itemize}

As this problem is overparameterized (there is a linear space of valid solutions), we need additional hypotheses regarding $\bW$. A common assumption is that $\bW$ is low-rank \cite{recht2010guaranteed}.

\end{example}

\begin{example}[Multi-class Linear Classification]\label{ex: multi-class}
Given $n$ data points $\bx_1,\dots,\bx_n \in \bbR^d$ with one-hot labels $\bY_1,\dots, \bY_n \in \bbR^k$:
\begin{itemize}
    \item $\bY_i = \be_{c(i)}$ where $c(i) \in [k]$ is the class of point $i$.
    \item $\calA(\bW)_{ij} = \bx_i^T\bW^{(j)}$ computes the prediction for class $j$.
    \item The constraint $\calA(\bW) = (\bY_1,\dots,\bY_n)$ ensures perfect interpolation.
\end{itemize}
\end{example}

\subsection{Optimization Framework}

We propose to minimize the empirical risk with a novel algorithm, which we call Matrix Stochastic Mirror Descent, that can be described as follows:

\begin{definition}[Training Objective]
The loss function takes the form:
$$\calL(\bW) = \frac{1}{p}\sum_{i=1}^{p}\ell_{i} \left(\calA(\bW)_i-\bb_{i}\right)$$
where each $\ell_i: \bbR \to \bbR_+$ is a convex loss function.
\end{definition}

\begin{definition}[Matrix Stochastic Mirror Descent]
Given a strongly convex mirror potential $\psi: \bbR^{d \times k} \to \bbR$, the SMD update rule is:
$$\nabla \psi(\bW_t) = \nabla \psi(\bW_{t-1}) - \eta \nabla_{\bW} \calL_t(\bW_{t-1})$$
where $\calL_t$ is the loss on a random batch sampled at iteration $t$:
\begin{align}\label{eq: L_t}
\calL_t(\bW) = \frac{1}{B}\sum_{j=1}^{B}\ell_{i_j} \left( \calA(\bW)_{i_j}-\bb_{i_j}\right)
\end{align}
\end{definition}

\subsection{Mathematical Preliminaries}

Before stating our main results, we remind the key definitions on convexity:

\begin{definition}[Matrix Convexity Properties]
A function $f: \bbR^{d \times k} \rightarrow \bbR$ is:
\begin{enumerate}[]
    \item \textbf{Convex} if $f(\theta \bU + (1-\theta) \bV) \le \theta f(\bU) + (1-\theta) f(\bV)$ for all $\bU, \bV$ and $\theta \in [0,1]$
    \item \textbf{Strictly convex} if the inequality is strict for $\theta \in (0,1)$
    \item \textbf{$\mu$-strongly convex} if 
    $$f(\bV) \ge f(\bU) + \tr(\nabla f(\bU)^T (\bV - \bU)) + \frac{\mu}{2} \|\bU - \bV\|_F^2$$
\end{enumerate}
\end{definition}

We will also make extensive use of the following definition of Bregman Divergence:

\begin{definition}[Matrix Bregman Divergence]
For a strictly convex differentiable mirror function $\psi:\bbR^{d \times k} \rightarrow \bbR$, the Bregman divergence is:
$$D_{\psi} (\bU, \bV) = \psi (\bU) - \psi(\bV) - \tr(\nabla \psi(\bV)^T (\bU - \bV))$$
\end{definition}

The Schatten $p$-norms will serve as the main illustrating example of the matrix mirrors for us:

\begin{definition}[Schatten Norm]
The Schatten $p$-norm of $\bW \in \bbR^{d \times k}$ is:
$$\|\bW\|_{\text{Schatten}, p} = \left(\sum_{i=1}^{\min(k,d)}\sigma_i(\bW)^p\right)^{\frac{1}{p}}$$
where $\sigma_1(\bW) \ge \cdots \ge \sigma_{\min(k,d)}(\bW) \ge 0$ are the singular values.
\end{definition}

\section{Main Results and Applications}

We will require the following assumptions. Note that, contrary to most other works in the literature, we do not require the $L$-smoothness condition, thus relaxing the common assumptions even for the case of vector weights. 

\begin{assumptions}\label{ass: main}
\begin{enumerate}
    \item      The mirror $\psi: \bbR^{d \times k} \to \bbR$ is differentiable and $\nu$-strongly convex for a $\nu > 0$.

  \item   The training loss is of the form $$\calL(\bW) =  \frac{1}{p} \sum_{i=1}^{p}\ell_{i}  \left({\calA(\bW)}_i -\bb_{i}\right)$$

    Moreover, $\ell_{i}$ is non-negative, has minimum $\ell_{i}(0) = 0$, has a derivative $\ell_i'$ which is continuous at $0$ and is $\mu$-strictly convex for a $\mu > 0$. 

   \item   The batches chosen for \eqref{eq: L_t} are chosen in such a way that $\bbE \calL_t = \calL$, where the expectation is taken with respect to the randomness in the choice of the batch. 

   \item $\eta > 0$ is small enough, so that $\psi - \eta \calL_t$ is convex. 
   
   \item   The matrix  $\bA = \begin{pmatrix}
        \ba_1, \ba_2 ,\dots,\ba_p
    \end{pmatrix}^T \in \bbR^{(d \times k) \times p}$ satisfies $\sigma_p(\bA) > 0$. Note that, in particular, this implies that we are in the overparameterized regime, i.e. that $p < dk$.  
       \item Let $\bW_\ast$ be the unique minimizer of the following optimization problem:
       \begin{align}
        \min_{\bW} D_{\psi}(\bW,\bW_0) \text{, s.t. }\calA(\bW) = \bb
    \end{align}
       Denote $$\calB = \{\bW: D_\psi(\bW_*, \bW) \le  D_{\psi}(\bW_*, \bW_0)\}$$
    Then there exists a $C>0$, such that the following holds for all $\bW \in \calB$: \begin{align*}
        \|\nabla^2 \psi(\bW)\|_{op} \le C
    \end{align*}
\end{enumerate}

\end{assumptions}
    
We are now ready to state our main result  characterizing the implicit bias and the convergence rate of Matrix SMD:
\begin{theorem}
    [Convergence Rate and the Implicit Bias] Assume that the linear operator $\calA: \bbR^{d \times k} \to \bbR^p$, the mirror $\psi: \bbR^{d \times k} \to \bbR$ and the training losses $\calL_t: \bbR^{d \times k} \to \bbR$ satisfy assumptions 1-4 from the list of Assumptions \ref{ass: main}, whose notation is also embraced below.  Introduce $\bW_* \in \bbR^{d \times k}$ as the unique minimizer of the following objective:
    \begin{align}\label{opt:impl}
        \min_{\bW} D_{\psi}(\bW,\bW_0) \text{, s.t. }\calA(\bW) = \bb
    \end{align}
     
    Denote the $t$-th iteration of the SMD  algorithm defined via \eqref{eq: L_t} with mirror $\psi$ trained to minimize $\calL(\bW)$ and initialized at $\bW_0$ by $\bW_t$. 
    Then $\bW_t$ converges to $\bW_*$ as $t \to \infty$. Assume, in addition, that assumption 5 and 6 from the list of Assumptions \ref{ass: main} holds as well and  let $$L := \max_{\bU, \bV \in \calB }\frac{D_\psi(\bU,\bV)}{\|\bU - \bV\|_F^2}$$ Then $\bW_t$ converges to $\bW_*$ exponentially, namely  the following holds:
    \begin{align}\label{eq: exp_conv}
        \bbE\|\bW_* -\bW_t\|_F^2 \le \frac{2}{\nu}\left(1-\frac{\eta\mu\sigma_p(\bA)^2}{2p L}\right)^tD_\psi\left(\bW_*, \bW_0\right)
    \end{align}
    Here, the expectation is taken with respect to the randomness in the batch choice at the step $i$ for all $i =1,\dots, t$. 
\end{theorem}

\begin{remark}
    To ensure \eqref{eq: exp_conv} implies exponential convergence, one must verify that the value of $L$ is finite. This is true as $\calB$ is compact and $\frac{D_\psi(\bU,\bV)}{\|\bU - \bV\|_F^2}$ remains bounded as $\bU \to \bV$ because of the assumption 6) from the list of Assumptions \ref{ass: main}. 
\end{remark}

\begin{remark}
Assumption 5 always holds for the matrix completion problem from Example \ref{ex: mtrx_compl}. For the multi-class linear classification from Example \ref{ex: multi-class}, we have $\bA = \bX \otimes \bI_k$ and  assumption 5 holds if and only if $\sigma_n(\bX) > 0$. 
\end{remark}

\begin{example}
    $\psi(\bW) = \|\bW\|^p_{\text{Schatten}, p} + \nu\|\bW\|_F^2$ satisfies  all of the Assumptions \ref{ass: main} if $p \ge 2$. Hence, we can deduce exponential convergence in this case. 
\end{example}
\begin{example}
As for $2>p>1$, $\psi(\bW) = \|\bW\|^p_{\text{Schatten}, p} + \nu\|\bW\|_F^2$ satisfies assumptions 1-5 from the list of the Assumptions \ref{ass: main}. It also satisfies assumption 6  if and only if $\calB$ does not contain singular matrices. Thus, in general, we can deduce convergence for $2>p>1$ but cannot specify the rate. 
\end{example}

\section{Proofs}

\subsection{Proof of Convergence}

The following lemma is a matrix analog of Lemma 4.1 from \cite{beck2003mirror}:
\begin{lemma}\label{lm: beck2003}
    For any $\bS,\bU,\bV \in \bbR^{d \times k}$ and any $f: \bbR^{d \times k} \to \bbR$, the following identity holds:
    \begin{align*}
        & D_f(\bV, \bS) +  D_f(\bS, \bU) - D_f(\bV, \bU) = \\
        & \tr{\left((\nabla f(\bU) - \nabla f(\bS))^T(\bV -\bS)\right)}
    \end{align*}
\end{lemma}

\begin{proof}
    The proof is identical to the proof of Lemma 4.1 from \cite{beck2003mirror}. 
\end{proof}

\begin{lemma}\label{lm: key_smd_identity}
    The following identity holds for any $\bW$ satisfying $\calA(\bW) = \bb$, where $\bW_i$ denote the iterates of the SMD algorithm, stochastic  loss functions $\calL_i: \bbR^{d \times k} \to \bbR$ and a learning rate $\eta>0$ satisfying Assumptions \ref{ass: main}. 
\begin{align*}
     & D_{\psi}(\bW,\bW_{i-1})   =  \eta D_{\calL_i}(\bW,\bW_{i-1}) +  \\ & D_\psi(\bW,\bW_i) + D_{\psi-\eta \calL_i}(\bW_i,\bW_{i-1}) +  \eta \calL_i(\bW_i)
\end{align*}
\end{lemma}
\begin{proof}   
Take an arbitrary $\bW \in \bbR^{d \times k}$. Using Lemma \ref{lm: beck2003}: 
\begin{align*}
     & D_{\psi}(\bW,\bW_i) + D_{\psi}(\bW_i,\bW_{i-1}) - D_{\psi}(\bW,\bW_{i-1})  \\=&  \tr{\left((\nabla \psi(\bW_{i-1}) - \nabla \psi(\bW_i))^T(\bW -\bW_i)\right)}
\end{align*}
Incorporating the definition of the SMD update, we arrive at:
\begin{align}\label{eq: psi_interm}
     & D_{\psi}(\bW,\bW_i) + D_{\psi}(\bW_i,\bW_{i-1}) - D_{\psi}(\bW,\bW_{i-1})  \\= & \eta\tr{\left((\nabla \calL_i(\bW_{i-1}))^T(\bW -\bW_i)\right)}
\end{align}
We also have from Lemma \ref{lm: beck2003} applied to $f = \calL_i$:
\begin{align}\label{eq: loss_interm}
     & \eta D_{\calL_i}(\bW,\bW_i) + \eta D_{\calL_i}(\bW_i,\bW_{i-1}) - \eta D_{\calL_i}(\bW,\bW_{i-1}) \nonumber \\ = & \eta\tr{\left((\nabla \calL_i(\bW_{i-1}) - \nabla \calL_i(\bW_i))^T(\bW -\bW_i)\right)}
\end{align}
Subtracting \eqref{eq: loss_interm} from \eqref{eq: psi_interm} we obtain: 
\begin{align}\label{eq: subtr_result}
     & D_{\psi -\eta \calL_i}(\bW,\bW_i) + D_{\psi-\eta \calL_i}(\bW_i,\bW_{i-1}) \nonumber \\ & - D_{\psi-\eta \calL_i}(\bW,\bW_{i-1})   = \eta\tr{\left(\nabla \calL_i(\bW_i))^T(\bW -\bW_i)\right)}
\end{align}
Equation \eqref{eq: subtr_result} is equivalent to:
\begin{align*}
     & D_{\psi}(\bW,\bW_{i-1})   =  \eta D_{\calL_i}(\bW,\bW_{i-1}) + D_{\psi -\eta \calL_i}(\bW,\bW_i)  \nonumber \\& +D_{\psi-\eta \calL_i}(\bW_i,\bW_{i-1}) -  \eta\tr{\left(\nabla \calL_i(\bW_i))^T(\bW -\bW_i)\right)}
\end{align*}
Opening up the $D_{\psi -\eta \calL_i}(\bW,\bW_i)$ term further:
\begin{align*}
     & D_{\psi}(\bW,\bW_{i-1})   =  \eta D_{\calL_i}(\bW,\bW_{i-1}) + D_\psi(\bW,\bW_i) \\ &  - \eta D_{\calL_i}(\bW,\bW_i) 
      +D_{\psi-\eta \calL_i}(\bW_i,\bW_{i-1}) \\& -  \eta\tr{\left(\nabla \calL_i(\bW_i))^T(\bW -\bW_i)\right)}
\end{align*}
Grouping the terms having $\eta$ in front:
\begin{align*}
     & D_{\psi}(\bW,\bW_{i-1})   =  \eta D_{\calL_i}(\bW,\bW_{i-1}) + D_\psi(\bW,\bW_i) \\
     & +D_{\psi-\eta \calL_i}(\bW_i,\bW_{i-1})   \\
     & -\eta \left(D_{\calL_i}(\bW,\bW_i) +\tr{\left(\nabla \calL_i(\bW_i))^T(\bW -\bW_i)\right)}\right)
\end{align*}
By definition of $D_{\calL_i}(\bW,\bW_i)$ we arrive at:
\begin{align*}
     & D_{\psi}(\bW,\bW_{i-1})   =  \eta D_{\calL_i}(\bW,\bW_{i-1}) + D_\psi(\bW,\bW_i)  \\
     &  +D_{\psi-\eta \calL_i}(\bW_i,\bW_{i-1}) -  \eta \left(\calL_i(\bW) - \calL_i(\bW_i)\right)
\end{align*}

We are now prepared to show convergence:
\begin{proof}
Assuming $\bW$ interpolates all data, i.e. $\calL_i(\bW) = 0$ for all $i$, we obtain the matrix analog of Lemma 6 from \cite{azizan2021stochastic} for any $\bW$ from the interpolating manifold:
\begin{align*}
     & D_{\psi}(\bW,\bW_{i-1})   =  \eta D_{\calL_i}(\bW,\bW_{i-1})  \\ & + D_\psi(\bW,\bW_i) + D_{\psi-\eta \calL_i}(\bW_i,\bW_{i-1}) +  \eta \calL_i(\bW_i)
\end{align*}
 \end{proof}

We are ready to prove convergence now. Note that Lemma \ref{lm: key_smd_identity} implies that 
$$D_{\psi}(\bW, \bW_{t-1}) \ge D_{\psi}(\bW, \bW_{t}) + \eta \calL_i(\bW_i)$$ for all $i=1,\dots, T$. Summing over $i=1,\dots,T$, we have: 
$$\sum_{i=1}^T D_{\psi}(\bW, \bW_{t-1}) \ge \sum_{i=1}^T  D_{\psi}(\bW, \bW_{t}) + \eta \sum_{i=1}^T  \calL_i(\bW_i)$$
Hence, $$D_{\psi}(\bW, \bW_{0}) \ge D_{\psi}(\bW, \bW_{T}) + \eta \sum_{i=1}^T  \calL_i(\bW_i)$$
Therefore, we see that $\calL_T(\bW_T) \to 0$ as $T \to \infty$, implying that $\calA(\bW_T) \to \bb$ as $T \to \infty$ and thus all SMD updates $\nabla \calL_T(\bW_T) \to 0$ as well, implying convergence to some point $\bW_{\infty}$.

\subsection{Implicit Bias}
Summing up the SMD iterations, we note that
\begin{align*}
    \nabla \psi(\bW_t) - \nabla \psi(\bW_0) =  \sum_{s=1}^t \eta  \nabla \calL_s (\bW_s) 
\end{align*}
Consider the following optimization problem, whose solution $\bW$ is unique due to strong convexity:
\begin{align*}
    \min_{\bW \in \bbR^{d \times k}} D_{\psi} (\bW, \bW_0) \\
    s.t \quad \calA(\bW) = \bb
\end{align*}
Using a Lagrange multiplier $\bm{\lambda} \in \bbR^p$:
\begin{align*}
    \min_{\bW \in \bbR^{d \times k}} \max_{\blambda \in \bbR^{p}}D_{\psi} (\bW, \bW_0) + \blambda^T (\calA(\bW) - \bb)
 \end{align*}
We compute the stationary conditions of KKT (the solution to this is unique as well):
\begin{align}\label{eq: KKT}
    \begin{cases}
        \nabla \psi(\bW) - \nabla \psi(\bW_0) = \bA^T\blambda \\
      \calA(\bW) = \bb
    \end{cases}
\end{align}
where $\bA$ is the matrix representation of $\calA$ such that $\calA(\bW) = \bA \text{vec}(\bW)$.

Now considering $$\calL_s(\bW):= \frac{1}{B}\sum_{ij=1}^{B}\ell_{i_j}  \left(\calA(\bW)_{i_j}-\bb_{i_j}\right),$$ 
we have for the SMD iterations:
\begin{align*}
    & \nabla \psi(\bW_t) - \nabla \psi(\bW_0) = \eta \sum_{s=1}^t \nabla \calL_s(\bW_s) \\
    &= \eta \sum_{s=1}^t \frac{1}{B}\sum_{j=1}^{B}\ell'_{i_j}  \left(\calA(\bW_s)_{i_j}-\bb_{i_j}\right) \nabla \calA(\bW)_{i_j} = \bA^T \bmu_t
\end{align*}
for some $\bmu_t \in \bbR^{p}$ because every $\nabla \calA(\bW)_{i_j}$ belongs in the span of $\ba_1,\dots,\ba_p$.

Now, assume that the SMD with constant step-size converges to some point $\bW_{\infty} \in \bbR^{d \times k}$. 
This implies that $\nabla \calL(\bW_{\infty}) = 0$, which by assumption implies $\calA(\bW_{\infty}) = \bb$. We also observe that $\nabla \psi(\bW_t) - \nabla \psi(\bW_0) = \bA^T \bmu_t$ for all $t \in \bbN$. Thus, taking $\blambda := \bmu_\infty$, we observe that $\bW_{\infty}$ satisfies the KKT conditions \eqref{eq: KKT}, which are assumed to yield a unique solution $\bW_*$.

\end{proof}

\subsection{Convergence Rate}

The proof below was inspired by the proofs provided in \cite{Varma2025Exponential} and \cite{d2023stochastic}. 

\begin{lemma}\label{lm: conv_rate_start}
    The following holds:
    \begin{align*}
        D_{\psi}(\bW_*, \bW_{t-1}) \ge \eta D_{\calL_t}(\bW_*, \bW_{t-1}) + D_{\psi}(\bW_*, \bW_{t})
    \end{align*}
\end{lemma}
\begin{proof}
Follows from 
\begin{align*}
     &D_{\psi}(\bW,\bW_{i-1})   =  \eta D_{\calL_i}(\bW,\bW_{i-1}) + D_\psi(\bW,\bW_i) \\ & + D_{\psi-\eta \calL_i}(\bW_i,\bW_{i-1}) +  
      \eta \calL_i(\bW_i) 
\end{align*}  
\end{proof}

\begin{lemma}
    Let $\bW_* - \bW_{t-1}= \bP +   \bP^\perp$, where $\text{vec}(\bP) \in \text{range}(\bA^T)$ and $\bA\text{vec}(\bP^\perp) = 0$. Then $$\calA(\bW_{t-1} + \bP) = \bb$$
\end{lemma}
\begin{proof}
    Since $\calA(\bW_*) = \bb$ and $\bA\text{vec}(\bP^\perp) = 0$ by definition, we have $$\calA(\bW_{t-1} + \bP) = \calA(\bW_* - \bP^\perp) = \bb$$ 
\end{proof}

\begin{lemma}
    The following holds:
    \begin{align*}
        \left(1-\frac{\eta\mu\sigma_p(\bA)^2}{2pL}\right)\bbE D_{\psi}(\bW_*, \bW_{t-1}) \ge  \bbE D_{\psi}(\bW_*, \bW_{t})
    \end{align*}
\end{lemma}

\begin{proof}
    It suffices to show the following due to Lemma \ref{lm: conv_rate_start}:
    \begin{align*}
        \eta\bbE D_{\calL_t}(\bW_*, \bW_{t-1}) \ge \frac{\eta\mu\sigma_p(\bA)^2}{2pL}D_{\psi}(\bW_*, \bW_{t-1})
    \end{align*}
    Since the expectation is taken over the randomness in the SMD batch, the latter is equivalent to:
     \begin{align*}
        D_{\calL}(\bW_*, \bW_{t-1}) \ge \frac{\mu\sigma_p(\bA)^2}{2pL}D_{\psi}(\bW_*, \bW_{t-1})
    \end{align*}

By strong convexity of $\ell_{i}$, $D_{\calL}(\bW_*, \bW_{t-1}) = $
     \begin{align*}
        &   \frac{1}{p}\sum_{i=1}^p  
         D_{\ell_{i}}(\calA(\bW_*)_i -\bb_i, \calA(\bW_{t-1})_i - \bb_i)  \ge \\
         &  \frac{1}{p}\sum_{i=1}^p  
         \frac{\mu}{2} (\calA(\bW_*)_i-\calA(\bW_{t-1})_i)^2 \\
         & =  \frac{\mu}{2p} \|\calA(\bW_*)-\calA(\bW_{t-1})\|_2^2  \\ & 
         = \frac{\mu}{2p} \|\bA\textbf{vec}(\bP)\|_2^2  \ge \frac{\mu\sigma_p(\bA)^2}{2p} \|\bP \|_F^2   \\ 
         & = \frac{\mu\sigma_p(\bA)^2}{2p} \|\bP + \bW_{t-1} - \bW_{t-1} \|_F^2  \\
         & \ge \frac{\mu\sigma_p(\bA)^2}{2pL}D_{\psi}(\bP + \bW_{t-1}, \bW_{t-1})  \\
         & \ge \frac{\mu\sigma_p(\bA)^2}{2pL}D_{\psi}(\bW_*, \bW_{t-1}) 
    \end{align*}
    Note that in the last line above we used that $\calA(\bP + \bW_{t-1}) = \bb$ and that $\bW_*$ minimizes $D_{\psi}(\bW, \bW_{t-1})$ with respect to the constraint $\calA(\bW) = \bb$. 
\end{proof}

\section{Experimental Setup}

We consider the problem of recovering a low-rank matrix $\bM \in \bbR^{n \times m}$ from a subset of its entries. The true matrix is generated as $\bM = \bU\bV^T$ where $\bU \in \bbR^{n \times r}$ and $\bV \in \bbR^{m \times r}$ have i.i.d. Gaussian entries scaled by $1/\sqrt{r}$, ensuring rank $r$. We observe each entry independently with probability $\textit{prob}$, yielding the observation set $\Omega$ and the partially observed matrix $\bM_\Omega$.

Since the common assumption regarding $\bM$ is low-rankness, approaches to the matrix completion problem usually solve the following objective:
\begin{align}
\min_{\bW} \quad & \|\bW\|_* \\
\text{s.t.} \quad & \bW_{ij} = \bM_{ij}, \quad (i,j) \in \Omega \nonumber
\end{align}
where $\|\cdot\|_*$ denotes the nuclear norm (sum of singular values).

\subsection{Methods}

We compare three algorithms for low-rank matrix completion driven by singular value shrinkage.

\begin{definition}[Singular Value Soft-Thresholding]
For a matrix with singular value decomposition $\bW = \bU \textbf{diag}(\sigma)\bV^T$  the soft-thresholding operator $\mathcal{S}_\tau$ is defined as:
$$\mathcal{S}_\tau(\bW) = \bU \cdot \textbf{diag}(\max(\sigma_i - \tau, 0)) \cdot \bV^T$$
where $\tau > 0$ is the threshold parameter.
\end{definition}

\textbf{1. Singular Value Thresholding (SVT) \cite{cai2010singular}} maintains an auxiliary matrix $\bY_t$ and iterates
\begin{align}
\bW_t &= \mathcal{S}_\tau(\bY_{t-1}) \\
\bY_t &= \bY_{t-1} + \delta \, \mathcal{P}_{\Omega}(\bM - \bW_t),
\end{align}
where $\mathcal{P}_{\Omega}$ is the projection onto observed entries, i.e. $\mathcal{P}_{\Omega}$ keeps the values for the entries in $\Omega$ and sets the rest to zero. 

\textbf{2. Soft-Impute \cite{mazumder2010spectral}} iterates as 
\begin{align}
(\bZ_t)_{ij} &= \begin{cases} 
\bM_{ij} & \text{if } (i,j) \in \Omega \\
(\bW_t)_{ij} & \text{otherwise}
\end{cases} \\
\bW_{t+1} &= \mathcal{S}_\lambda(\bZ_t),
\end{align}

\textbf{3. Schatten-$p$ Mirror Descent:} Our proposed method using the mirror $\psi(\bW) = \|\bW\|_{\text{Schatten}, p}^p$ with $p$ slightly above 1 (we use $p=1.05$). 
\subsection{Results}

We evaluate all methods on $100 \times 100$ matrices of rank $5$, varying the sampling probability from $0.1$ to $0.9$ in increments of $0.1$. Each method runs for 200 iterations.

For the SVT method we use step size $\delta = 0.8$; the SVT shrinkage parameter is set to the paper-style default $\tau = 5\max(n,m)$. For Soft-Impute we use $\lambda = 1.0$. For Schatten-$p$ SMD we use $p=1.05$ and learning rate $\eta = 50$, with the gradient normalized by $|\Omega|$.

Figure 1 shows the relative Frobenius norm error $\frac{\|\bW - \bM\|_F} {\|\bM\|_F}$ as a function of sampling probability. The Schatten-$p$ mirror descent consistently outperforms both thresholding methods across all sampling rates, with the advantage most pronounced at lower sampling probabilities where the problem is most challenging.

\begin{figure}[t]
  \centering
  \includegraphics[width=1.0\linewidth]{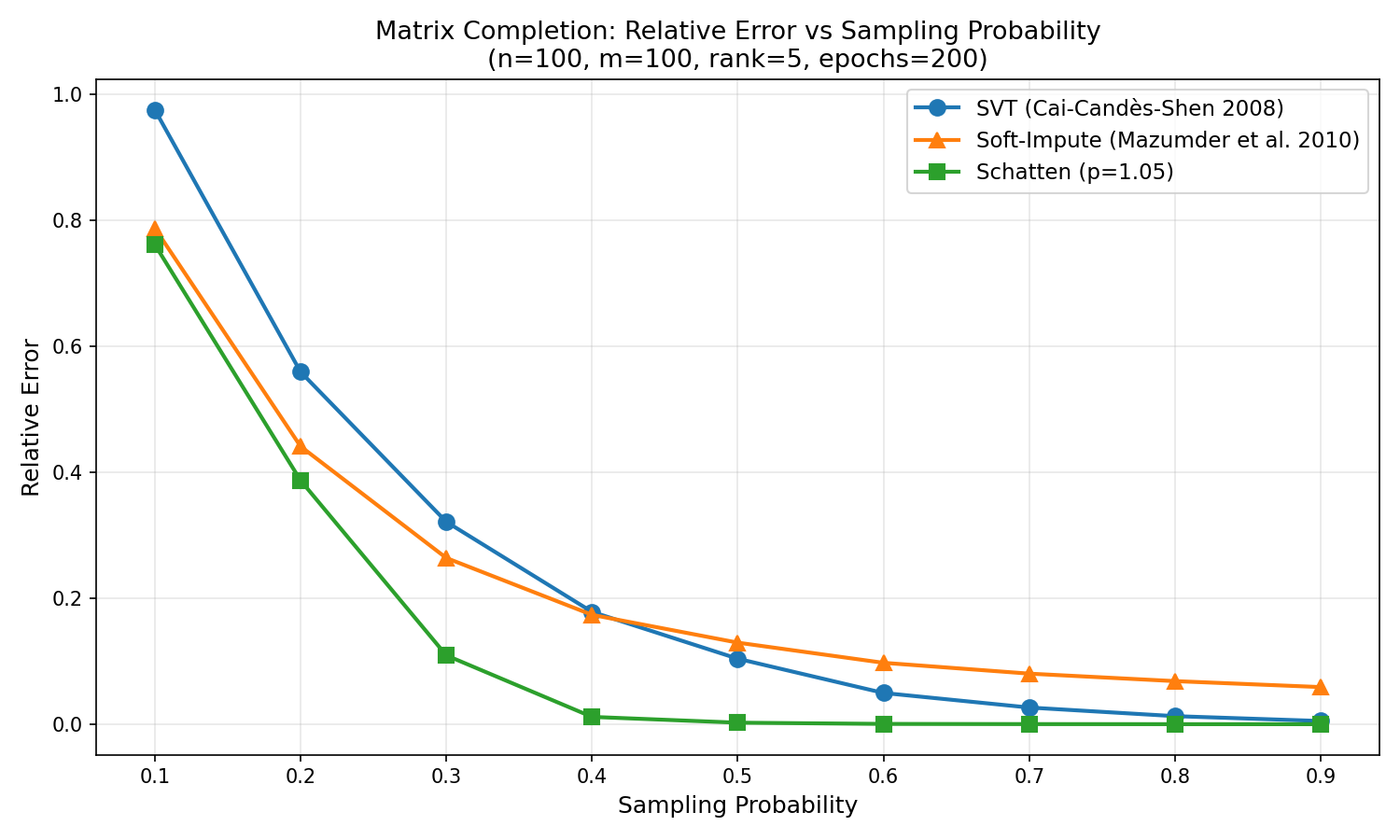}
  \caption{Relative recovery error versus sampling probability for SVT \cite{cai2010singular}, Soft-Impute \cite{mazumder2010spectral}, and Schatten-$p$ SMD.}
  \label{fig:completion_error_vs_sampling}
\end{figure}


\section{Conclusion}

This paper extends the theory of stochastic mirror descent to matrix parameters and vector-valued outputs, providing both theoretical guarantees and practical benefits. 

Our experiments on matrix completion validate the practical value of this framework. Schatten-$p$ mirror descent with $p \approx 1$  outperforms standard proximal methods based on singular value thresholding ran with the same number of epochs, particularly in challenging low-sampling regimes, by naturally inducing low-rank structure through the geometry of the mirror map rather than explicit constraints. 

While our analysis establishes convergence for SMD with $1 < p < 2$, proving exponential rates in this regime requires relaxing assumption 6 in the list of Assumptions \ref{ass: main}. This remains an important direction for future work.

\newpage

\bibliographystyle{IEEEtran}
\bibliography{sample}

\end{document}

%% file: symbols.tex
\newcommand{\bGamma}{\bm{\Gamma}}
\newcommand{\bDelta}{\bm{\Delta}}
\newcommand{\bLambda}{\boldsymbol{\Lambda}}
\newcommand{\bSigma}{\mathbf{\Sigma}}
\newcommand{\bOmega}{\bm{\Omega}}
\newcommand{\bPsi}{\bm{\Psi}}
\newcommand{\bPhi}{\bm{\Phi}}
\newcommand{\tTheta}{\tilde{\Theta}}
\newcommand{\blambda}{\boldsymbol{\lambda}}

\newcommand{\clip}{\text{clip}}
\newcommand{\sign}{\text{sign}}
\newcommand{\prox}{\text{prox}}

\newcommand{\rk}[1]{\operatorname{rk}(#1)}
\newcommand{\hA}{\hat{\bA}} 
\newcommand{\balpha}{\bm{\alpha}}
\newcommand{\bbeta}{\bm{\beta}}
\newcommand{\bdelta}{\bm{\delta}}
\newcommand{\bomega}{\bm{\omega}}
\newcommand{\bgamma}{\bm{\gamma}}
\newcommand{\tbOmega}{\tilde{\bOmega}}
\newcommand{\bepsilon}{\bm{\epsilon}}
\newcommand{\btheta}{\bm{\theta}}
\newcommand{\bphi}{\bm{\phi}}
\newcommand{\bvarphi}{\bm{\varphi}}
\newcommand{\bpsi}{\bm{\psi}}
\newcommand{\bmeta}{\bm{\eta}}
\newcommand{\bzeta}{\bm{\zeta}}
\newcommand{\bmu}{\bm{\mu}}
\newcommand{\bnu}{\bm{\nu}}
\newcommand{\bpi}{\bm{\pi}}
\newcommand{\bsigma}{\mathbf{\sigma}}
\newcommand{\boldeta}{\boldsymbol{\eta}}
\newcommand{\tbx}{\bm{\tilde{x}}}
\newcommand{\tbw}{\bm{\tilde{w}}}
\newcommand{\tbS}{\bm{\tilde{S}}}

\newcommand{\bargamma}{\bar{\gamma}}
\newcommand{\bartheta}{\bar{\theta}}
\newcommand{\barphi}{\bar{\phi}}
\newcommand{\tSigz}{\tilde{\Sigma}_Z}

\newcommand{\tilDelta}{\tilde{\Delta}}
\newcommand{\tlDelta}{\tilde{\Delta}}
\newcommand{\tlPhi}{\tilde{\Phi}}

\newcommand{\tlepsilon}{\tilde{\epsilon}}
\newcommand{\tldelta}{\tilde{\delta}}
\newcommand{\tltheta}{\tilde{\theta}}
\newcommand{\tlgamma}{\tilde{\gamma}}
\newcommand{\tlpsi}{\tilde{\psi}}
\newcommand{\tlrho}{\tilde{\rho}}

\newcommand{\bA}{\mathbf{A}}
\newcommand{\bB}{\mathbf{B}}
\newcommand{\bC}{\mathbf{C}}
\newcommand{\bD}{\mathbf{D}}
\newcommand{\bE}{\mathbf{E}}
\newcommand{\bF}{\mathbf{F}}
\newcommand{\bG}{\mathbf{G}}
\newcommand{\bH}{\mathbf{H}}
\newcommand{\bI}{\mathbf{I}}
\newcommand{\bJ}{\mathbf{J}}
\newcommand{\bL}{\mathbf{L}}
\newcommand{\bM}{\mathbf{M}}
\newcommand{\bN}{\mathbf{N}}
\newcommand{\bP}{\mathbf{P}}
\newcommand{\bQ}{\mathbf{Q}}
\newcommand{\bR}{\mathbf{R}}
\newcommand{\bS}{\mathbf{S}}
\newcommand{\bT}{\mathbf{T}}
\newcommand{\bU}{\mathbf{U}}
\newcommand{\bV}{\mathbf{V}}
\newcommand{\bW}{\mathbf{W}}
\newcommand{\bX}{\mathbf{X}}
\newcommand{\bY}{\mathbf{Y}}
\newcommand{\bZ}{\mathbf{Z}}

\newcommand{\ba}{\mathbf{a}}
\newcommand{\bb}{\mathbf{b}}
\newcommand{\bc}{\mathbf{c}}
\newcommand{\bd}{\mathbf{d}}
\newcommand{\be}{\mathbf{e}}
\newcommand{\mbf}{\mathbf{f}}
\newcommand{\bg}{\mathbf{g}}
\newcommand{\bh}{\mathbf{h}}
\newcommand{\bl}{\mathbf{l}}
\newcommand{\bbm}{\bm{m}}
\newcommand{\bn}{\mathbf{n}}
\newcommand{\bp}{\mathbf{p}}
\newcommand{\bq}{\mathbf{q}}
\newcommand{\br}{\mathbf{r}}
\newcommand{\bs}{\mathbf{s}}
\newcommand{\bt}{\mathbf{t}}
\newcommand{\bu}{\mathbf{u}}
\newcommand{\bv}{\mathbf{v}}
\newcommand{\bw}{\mathbf{w}}
\newcommand{\bx}{\mathbf{x}}
\newcommand{\by}{\mathbf{y}}
\newcommand{\bz}{\mathbf{z}}

\newcommand{\bTheta}{\bm{\Theta}}

\newcommand{\hbeta}{\hat{\beta}}
\newcommand{\hepsilon}{\hat{\epsilon}}
\newcommand{\htheta}{\hat{\theta}}
\newcommand{\hsigma}{\hat{\sigma}}
\newcommand{\hmu}{\hat{\mu}}
\newcommand{\btau}{\bm{\tau}}
\newcommand{\bxi}{\bm{\xi}}

\newcommand{\hf}{\hat{f}}
\newcommand{\hp}{\hat{p}}
\newcommand{\hr}{\hat{r}}
\newcommand{\hs}{\hat{s}}
\newcommand{\hw}{\hat{w}}
\newcommand{\hx}{\hat{x}}

\newcommand{\hN}{\hat{N}}

\newcommand{\hbSigma}{\hat{\bm{\Sigma}}}

\newcommand{\hbA}{\hat{\mathbf{A}}}
\newcommand{\hba}{\hat{\mathbf{a}}}
\newcommand{\hbs}{\hat{\mathbf{s}}}
\newcommand{\hbx}{\hat{\mathbf{x}}}
\newcommand{\hbv}{\hat{\mathbf{v}}}
\newcommand{\hbw}{\hat{\mathbf{w}}}

\newcommand{\hbW}{\hat{\mathbf{W}}}

\newcommand{\dif}{\text{d}}

\newcommand{\bbC}{\mathbb{C}}
\newcommand{\bbE}{\mathbb{E}}
\newcommand{\bbR}{\mathbb{R}}
\newcommand{\bbS}{\mathbb{S}}
\newcommand{\bbP}{\mathbb{P}}
\newcommand{\bbN}{\mathbb{N}}
\newcommand{\bbZ}{\mathbb{Z}}
\newcommand{\bbH}{\mathbb{H}}
\newcommand{\bbQ}{\mathbb{Q}}

\newcommand{\calA}{\mathcal{A}}
\newcommand{\calB}{\mathcal{B}}
\newcommand{\calC}{\mathcal{C}}
\newcommand{\calD}{\mathcal{D}}
\newcommand{\calE}{\mathcal{E}}
\newcommand{\calF}{\mathcal{F}}
\newcommand{\calG}{\mathcal{G}}
\newcommand{\calH}{\mathcal{H}}
\newcommand{\calI}{\mathcal{I}}
\newcommand{\calJ}{\mathcal{J}}
\newcommand{\calL}{\mathcal{L}}
\newcommand{\calN}{\mathcal{N}}
\newcommand{\calM}{\mathcal{M}}
\newcommand{\calP}{\mathcal{P}}
\newcommand{\calR}{\mathcal{R}}
\newcommand{\calS}{\mathcal{S}}
\newcommand{\calT}{\mathcal{T}}
\newcommand{\calV}{\mathcal{V}}
\newcommand{\calW}{\mathcal{W}}
\newcommand{\calX}{\mathcal{X}}
\newcommand{\calY}{\mathcal{Y}}

\newcommand{\mscl}{\mathscr{\ell}}
\newcommand{\mscm}{\mathscr{m}}

\newcommand{\calhL}{\mathcal{\hat{L}}}
\newcommand{\bcalP}{\bm{\calP}}

\newcommand{\tlA}{\tilde{A}}
\newcommand{\tlC}{\tilde{C}}
\newcommand{\tlD}{\tilde{D}}
\newcommand{\tlR}{\tilde{R}}

\newcommand{\tla}{\tilde{a}}
\newcommand{\tlc}{\tilde{c}}
\newcommand{\tlf}{\tilde{f}}
\newcommand{\tlg}{\tilde{g}}
\newcommand{\tlv}{\tilde{v}}
\newcommand{\tls}{\tilde{s}}
\newcommand{\tlw}{\tilde{w}}
\newcommand{\tlx}{\tilde{x}}
\newcommand{\tly}{\tilde{y}}
\newcommand{\tlz}{\tilde{z}}
\newcommand{\tbSigma}{\tilde{\bSigma}}
\newcommand{\tbmu}{\tilde{\bmu}}
\newcommand{\tbX}{\tilde{\bX}}

\newcommand{\barb}{\bar{b}}
\newcommand{\barm}{\bar{m}}
\newcommand{\barn}{\bar{n}}
\newcommand{\barr}{\bar{r}}
\newcommand{\barv}{\bar{v}}
\newcommand{\barx}{\bar{x}}
\newcommand{\bary}{\bar{y}}
\newcommand{\barz}{\bar{z}}

\newcommand{\barA}{\bar{A}}
\newcommand{\barC}{\bar{C}}
\newcommand{\barD}{\bar{D}}
\newcommand{\barH}{\bar{H}}
\newcommand{\barK}{\bar{K}}
\newcommand{\barL}{\bar{L}}
\newcommand{\barV}{\bar{V}}
\newcommand{\barW}{\bar{W}}
\newcommand{\barX}{\bar{X}}
\newcommand{\barZ}{\bar{Z}}

\newcommand{\barba}{\bar{\ba}}
\newcommand{\barbe}{\bar{\be}}
\newcommand{\barbg}{\bar{\bg}}
\newcommand{\barbh}{\bar{\bh}}
\newcommand{\barbx}{\bar{\bx}}
\newcommand{\barby}{\bar{\by}}
\newcommand{\barbz}{\bar{\bz}}

\newcommand{\barbA}{\bar{\bA}}

\newcommand{\tlbA}{\tilde{\bA}}
\newcommand{\tlbB}{\tilde{\bB}}
\newcommand{\tlbD}{\tilde{\bD}}
\newcommand{\tlbE}{\tilde{\bE}}
\newcommand{\tlbG}{\tilde{\bG}}
\newcommand{\tlbM}{\tilde{\bM}}

\newcommand{\tlbW}{\tilde{\bW}}
\newcommand{\tlbX}{\tilde{\bX}}
\newcommand{\tlbY}{\tilde{\bY}}

\newcommand{\tlba}{\tilde{\ba}}
\newcommand{\tlbf}{\tilde{\mbf}}
\newcommand{\tlbg}{\tilde{\bg}}
\newcommand{\tlbv}{\tilde{\bv}}
\newcommand{\tlbw}{\tilde{\bw}}
\newcommand{\tlbx}{\tilde{\bx}}
\newcommand{\tlby}{\tilde{\by}}
\newcommand{\tlbz}{\tilde{\bz}}

\newcommand{\thetadot}{\Dot{\theta}}

\newcommand{\tc}{\text{c}}
\newcommand{\td}{{\text{d}}}
\newcommand{\ter}{{\text{r}}}
\newcommand{\ts}{{\text{s}}}
\newcommand{\tw}{{\text{w}}}

\newcommand{\bzero}{\mathbf{0}}
\newcommand{\bone}{\mathbf{1}}
\newcommand{\dsone}{\mathds{1}}

\newcommand{\suml}{\sum\limits}
\newcommand{\minl}{\min\limits}
\newcommand{\maxl}{\max\limits}
\newcommand{\infl}{\inf\limits}
\newcommand{\supl}{\sup\limits}
\newcommand{\liml}{\lim\limits}
\newcommand{\intl}{\int\limits}
\newcommand{\ointl}{\oint\limits}
\newcommand{\bigcupl}{\bigcup\limits}
\newcommand{\bigcapl}{\bigcap\limits}

\newcommand{\opconv}{\text{conv}}

\newcommand{\eref}[1]{(\ref{#1})}

\newcommand{\sinc}{\text{sinc}}
\newcommand{\tr}{\text{Tr}}
\newcommand{\diag}{\text{diag}}
\newcommand{\var}{\text{Var}}
\newcommand{\cov}{\text{Cov}}
\newcommand{\tth}{\text{th}}
\newcommand{\proj}{\mathrm{proj}}
\newcommand*\diff{\mathop{}\!\mathrm{d}}
\newcommand{\rarrowp}{\xrightarrow[]{\bbP}}
\newcommand{\rarrowd}{\xrightarrow[]{d}}
\newcommand{\allone}{\mathds{1}}
\newcommand{\norm}[1]{\left\|#1\right\|}
\newcommand{\iu}{{i\mkern1mu}}

\newcommand{\nwl}{\nonumber\\}

\newenvironment{vect}{\left[\begin{array}{c}}{\end{array}\right]}

\newcommand{\calO}{\mathcal{O}}
\newtheorem{theorem}{Theorem}
\newtheorem{Proposition}{Proposition}
\newtheorem{remark}{Remark}
\newtheorem{lemma}{Lemma}
\newtheorem{corollary}{Corollary}
\newtheorem{definition}{Definition}
\newtheorem{assumptions}{Assumptions}
\newtheorem{example}{Example}

%% file: sample.bib
@inproceedings{gunasekar2018characterizing,
  title={Characterizing implicit bias in terms of optimization geometry},
  author={Gunasekar, Suriya and Lee, Jason and Soudry, Daniel and Srebro, Nathan},
  booktitle={International Conference on Machine Learning},
  pages={1832--1841},
  year={2018},
  organization={PMLR}
}

@article{d2023stochastic,
  title={Stochastic Mirror Descent: Convergence Analysis and Adaptive Variants via the Mirror Stochastic Polyak Stepsize},
  author={D'Orazio, Ryan and Loizou, Nicolas and Laradji, Issam H and Mitliagkas, Ioannis},
  journal={Trans. Mach. Learn. Res.},
  year={2023}
}

@article{cai2010singular,
  title={A singular value thresholding algorithm for matrix completion},
  author={Cai, Jian-Feng and Cand{\`e}s, Emmanuel J and Shen, Zuowei},
  journal={SIAM Journal on optimization},
  volume={20},
  number={4},
  pages={1956--1982},
  year={2010},
  publisher={SIAM}
}

@article{mazumder2010spectral,
  title={Spectral regularization algorithms for learning large incomplete matrices},
  author={Mazumder, Rahul and Hastie, Trevor and Tibshirani, Robert},
  journal={The Journal of Machine Learning Research},
  volume={11},
  pages={2287--2322},
  year={2010},
  publisher={JMLR.org}
}

@article{beck2003mirror,
  title={Mirror descent and nonlinear projected subgradient methods for convex optimization},
  author={Beck, Amir and Teboulle, Marc},
  journal={Operations Research Letters},
  volume={31},
  number={3},
  pages={167--175},
  year={2003},
  publisher={Elsevier}
}

@article{recht2010guaranteed,
  title={Guaranteed minimum-rank solutions of linear matrix equations via nuclear norm minimization},
  author={Recht, Benjamin and Fazel, Maryam and Parrilo, Pablo A},
  journal={SIAM review},
  volume={52},
  number={3},
  pages={471--501},
  year={2010},
  publisher={SIAM}
}

@INPROCEEDINGS{Varma2025Exponential,
  author={Varma, K Nithin and Hassibi, Babak},
  booktitle={ICASSP 2025 - 2025 IEEE International Conference on Acoustics, Speech and Signal Processing (ICASSP)}, 
  title={Exponential Convergence of Stochastic Mirror Descent in Over-parameterized Linear Models}, 
  year={2025},
  volume={},
  number={},
  pages={1-5},
  doi={10.1109/ICASSP49660.2025.10888059}}

@article{azizan2021stochastic,
  title={Stochastic mirror descent on overparameterized nonlinear models},
  author={Azizan, Navid and Lale, Sahin and Hassibi, Babak},
  journal={IEEE Transactions on Neural Networks and Learning Systems},
  volume={33},
  number={12},
  pages={7717--7727},
  year={2021},
  publisher={IEEE}
}

@article{soudry2018implicit,
  title={The implicit bias of gradient descent on separable data},
  author={Soudry, Daniel and Hoffer, Elad and Nacson, Mor Shpigel and Gunasekar, Suriya and Srebro, Nathan},
  journal={Journal of Machine Learning Research},
  volume={19},
  number={70},
  pages={1--57},
  year={2018}
}

@article{sun2022mirror,
  title={Mirror descent maximizes generalized margin and can be implemented efficiently},
  author={Sun, Haoyuan and Ahn, Kwangjun and Thrampoulidis, Christos and Azizan, Navid},
  journal={Advances in Neural Information Processing Systems},
  volume={35},
  pages={31089--31101},
  year={2022}
}

@article{azizan2018stochastic,
  title={Stochastic gradient/mirror descent: Minimax optimality and implicit regularization},
  author={Azizan, Navid and Hassibi, Babak},
  journal={arXiv preprint arXiv:1806.00952},
  year={2018}
}

@inproceedings{ji2019implicit,
  title={The implicit bias of gradient descent on nonseparable data},
  author={Ji, Ziwei and Telgarsky, Matus},
  booktitle={Conference on learning theory},
  pages={1772--1798},
  year={2019},
  organization={PMLR}
}
